\def\year{2021}\relax
\documentclass[letterpaper]{article} % DO NOT CHANGE THIS
\usepackage{aaai21}  % DO NOT CHANGE THIS
\usepackage{times}  % DO NOT CHANGE THIS
\usepackage{helvet} % DO NOT CHANGE THIS
\usepackage{courier}  % DO NOT CHANGE THIS
\usepackage[hyphens]{url}  % DO NOT CHANGE THIS
\usepackage{graphicx} % DO NOT CHANGE THIS
\usepackage{natbib}  % DO NOT CHANGE THIS AND DO NOT ADD ANY OPTIONS TO IT
\usepackage{caption} % DO NOT CHANGE THIS AND DO NOT ADD ANY OPTIONS TO IT
\nocopyright
\usepackage[ruled, vlined]{algorithm2e}
\DontPrintSemicolon

\usepackage{amsmath,amssymb,amsfonts}

\usepackage{amsthm}

\theoremstyle{plain}
\newtheorem{lemma}{Lemma}
\newtheorem{theorem}{Theorem}

\theoremstyle{definition}

\usepackage[caption=false,font=normalsize,labelfont=sf,textfont=sf]{subfig}
\graphicspath{{figures/}}

\title{Multi-Agent Path Finding Based on Subdimensional Expansion with Bypass}
\author{Qingzhou Liu \quad Feng Wu \\
}
\affiliations{
    School of Computer Science and Technology,
    University of Science and Technology of China \\
    Jingzhai Road 96, Hefei, 230026, Anhui, China \\
    wufeng02@ustc.edu.cn
}

\begin{document}
%\linenumbers

\maketitle

\begin{abstract}
Multi-agent path finding (MAPF) is an active area in artificial intelligence, which has many real-world applications such as warehouse management, traffic control, robotics, etc. Recently, M* and its variants have greatly improved the ability to solve the MAPF problem. Although subdimensional expansion used in those approaches significantly decreases the dimensionality of the joint search space and reduces the branching factor, they do not make full use of the possible non-uniqueness of the optimal path of each agent. As a result, the updating of the collision sets may bring a large number of redundant computation. In this paper, the idea of bypass is introduced into subdimensional expansion to reduce the redundant computation. Specifically, we propose the BPM* algorithm, which is an implementation of subdimensional expansion with bypass in M*. In the experiments, we show that BPM* outperforms the state-of-the-art in solving several MAPF benchmark problems.
\end{abstract}

\section{Introduction}
\label{sec:introduction}

Multi-agent path finding (MAPF) is a planning problem to find a feasible path for each agent from its distinct start position to its distinct goal position without having any conflict with other agents along the path. MAPF can be applied to many real-world applications, such as warehouse management \cite{b1}, traffic control \cite{b2}, robotics \cite{b3}, etc.

Unfortunately, MAPF has been proved to be NP-complete \cite{b4}. The configuration space of MAPF grows exponentially with the number of agents in the problem. Therefore, the optimal MAPF algorithms are only suitable for solving problems which contain only a small number of agents. Sacrificing the optimality of these algorithms can greatly improve the solving speed and scale to problems with a large number of agents, which motivates several bounded suboptimal MAPF algorithms. However, the optimal MAPF algorithms still have a great research value because most of the suboptimal approaches are variants of the optimal algorithms. Indeed, there exists a tradeoff between the solving speed and the quality of solution if the requirement of optimality can be relaxed in practice.

Among the state-of-the-art optimal MAPF algorithms, subdimensional expansion is one of the lead techniques to search the solution space. Based on subdimensional expansion, M* \cite{b7} and its variants have shown greatly improvement on the ability to solve MAPF problems. Specifically, subdimensional expansion updates the searching space by maintaining collision sets. This decreases the dimensionality of the joint search space and reduces the branching factor. However, it does not make full use of the possible non-uniqueness of the optimal path of each agent. Therefore, the updating of its collision sets may bring many redundant computations.

In this paper, the idea of bypass is introduced into subdimensional expansion to develop a new method for solving MAPF. The main idea of subdimensional expansion with bypass is to try to avoid the conflicts among agents at first in the process of path planning. Then, the dimensionality of relevant sub-search spaces will be increased only when the relevant agents fail to bypass. When there exists no conflict, each agent will follow its individual optimal policy to move. Notably, subdimensional expansion with bypass inherits the completeness and optimality of subdimensional expansion. By the modification of the updating method of collision sets, the redundant computation is reduced.

On the basis of the aforementioned technique, we propose the BPM* algorithm for solving MAPF, which is an implementation of subdimensional expansion with bypass. Given the design of BPM*, we prove that it is complete and optimal. Note that completeness and optimality are two important properties in the study of MAPF algorithms. In the experiments, we compare BPM* with M* and their variants. The experimental results show that, both on the success rates orthe solving speed, our approach has better performance than the state-of-the-art methods in solving MAPF.

\section{Background}

Similar to M*, our method is essentially a variant of A* | a classical search algorithm that has been extended to solving MAPF. When using A* to solve a MAPF instance, the search space is often called $k-agent$ space. In the initial state, each agent stays in its start position, while in the terminal state, all agents will arrive in their goal positions. The MAPF algorithm tries to find a set of pathes, one for each agent from its initial state to the goal state without conflicts.

The vanilla A* has two drawbacks in solving MAPF. The first one is the exponential growth of the search space. Specifically, the search space of A* grows exponentially with the number of agents. Another drawback is the exponential growth of the branching factor. Since each agent can take $b(b\ge 1)$ actions at one time step, when there exist $n$ agents, the branching factor is equal to $b^n$. To improve the performance of A*, several improvements have been proposed in the MAPF literature. Next, we will briefly review four leading approaches based on A* as follows.

Operator Decomposition (OD) \cite{b5} improves the defect of exponential growth of the branching factor in A* search. When expanding the initial node $q_I$, OD only allows agent $a^1$ to move. Then $b$ successor nodes of $q_I$ are generated. In these successor nodes, only $a^1$ moves, the rest $n-1$ agents remain in the same positions as in the predecessor node. These successor nodes are then added into the {\it open-list}. When expanding these nodes, only agent $a^2$ are allowed to move, another $b$ successor nodes are generated. The new generated successor nodes represent all possible distribution of $a^1$ and $a^2$ at time step 1. It repeats these steps until the goal node $q_F$ is generated. In the search tree of OD, only nodes on levels whose number is divisible by $n$ can represent the distribution of all agents at the same time. Such nodes are called full nodes and the others are called intermediate nodes.

Enhanced Partial Expansion (EPE) \cite{b6} also improves the defect of exponential growth of the branching factor in the A* search. When expanding a node, it only generates a part of the successor nodes that A* generates. These generated nodes are then added into the {\it open-list}. In EPE, each node $q$ has an offset $f_{off}$, the initial value of $f_{off}$ is 1. When expanding $q$, only nodes whose heuristic function value is equal to $f(q)+f_{off}$ will be generated, after which, $f_{off}=f_{off}+1$. When there exists at least one node whose heuristic function value is greater than or equals to $f(q)+f_{off}$, $q$ will be added into the {\it open-list} again.

Independence Detection (ID) \cite{b5} improves the defect of exponential growth of the search space, which is essentially an independent detection among agents. It attempts to decompose a MAPF problem with $n$ agents into multiple subproblems with fewer agents. At first, each agent plans its own independent optimal path without considering the existence of other agents. Then it will detect the independence of the $n$ independent optimal pathes planned by $n$ agents. If there is a conflict between any two pathes, it will merge the corresponding agents into a group, and continue to detect and merge until there is no conflict between any two path groups. The detection then terminates and the path finding for each group of agents is solved as a subproblem.

M* \cite{b7} improves the two defects at the same time. It is also the building block of this work. The main idea of M* is subdimensional expansion, which will be described in details next.

\section{Subdimensional Expansion with Bypass}

In this section, we first introduce the basic concept of subdimension expansion. Based on it, we then propose our new method of adding bypass to it.

\subsection{Subdimensional Expansion}

Subdimension expansion is a general framework for solving the MAPF problem. In this framework, each agent $a^i$ owns its individual optimal policy $\phi^i:Q^i\rightarrow TQ^i$, which maps an agent's position to its move. For $a^i$, moving from any position $q_k^i\in Q^i$, if follows $\phi^i$, an optimal path to $q_F^i\in Q^i$ can be found, which can be denoted as $\pi_{\phi}^i(q_k^i,q_F^i)$.

\begin{figure}[t]
	\centering\small
	\subfloat[A MAPF instance]{
		\includegraphics[width=0.3\linewidth]{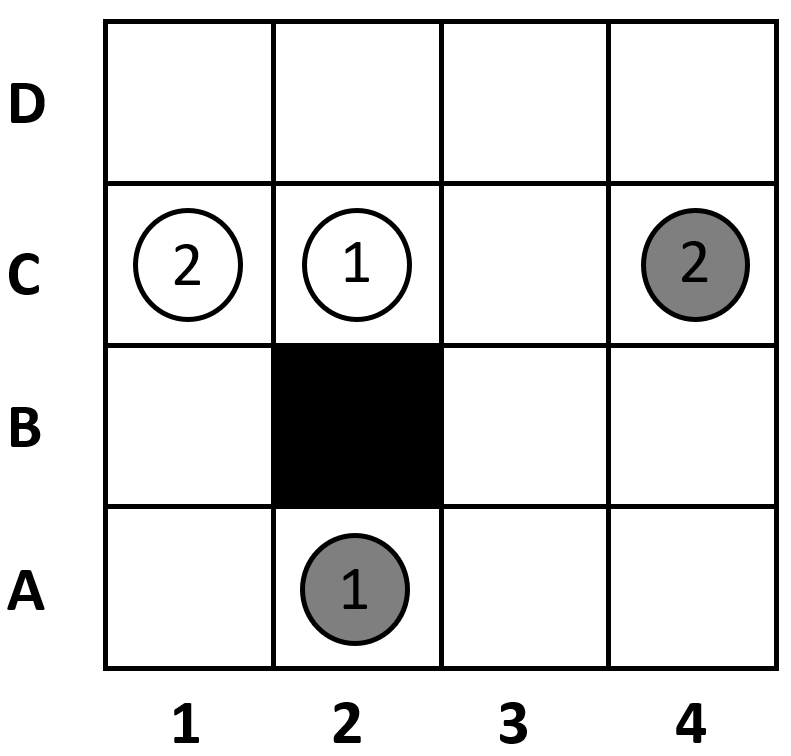}}
	\label{1a}\hfill
	\subfloat[Initial policy of Agent $a^1$]{
		\includegraphics[width=0.3\linewidth]{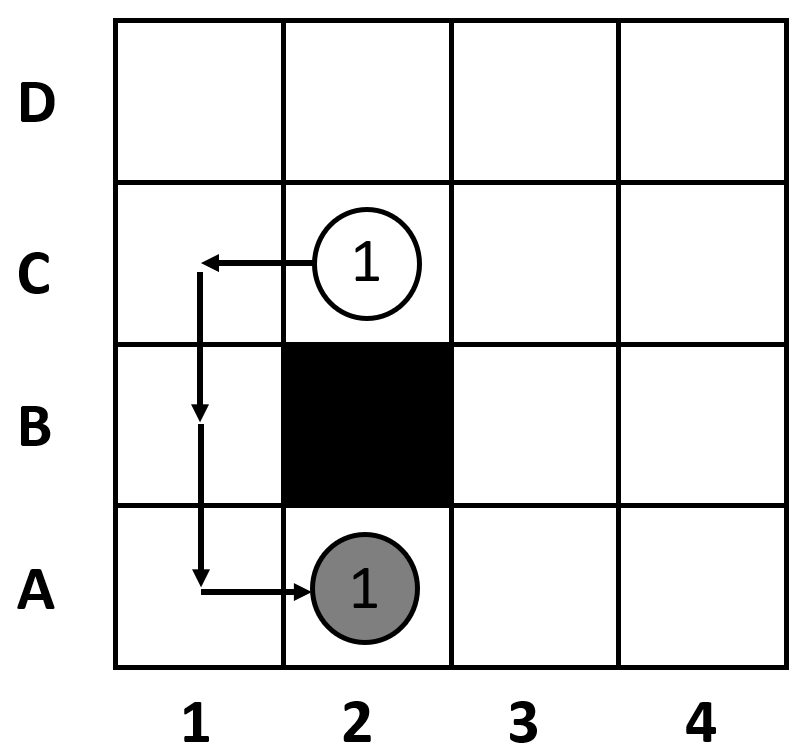}}
	\label{1b} \\
	\subfloat[Initial policy of Agent $a^2$]{
		\includegraphics[width=0.3\linewidth]{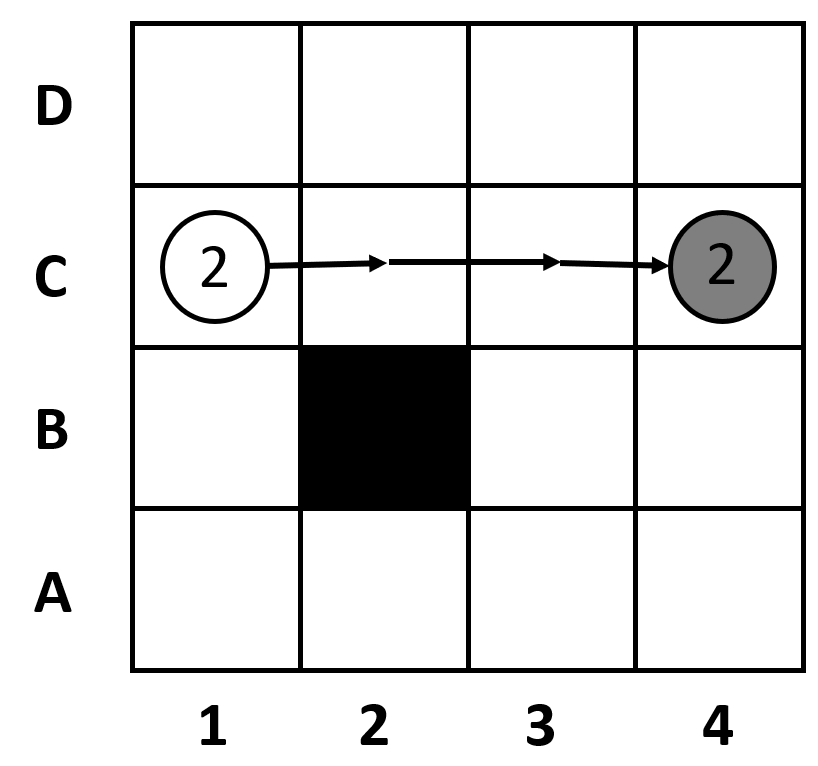}}
	\label{1c} \hfill
	\subfloat[Search tree]{
		\includegraphics[width=0.3\linewidth]{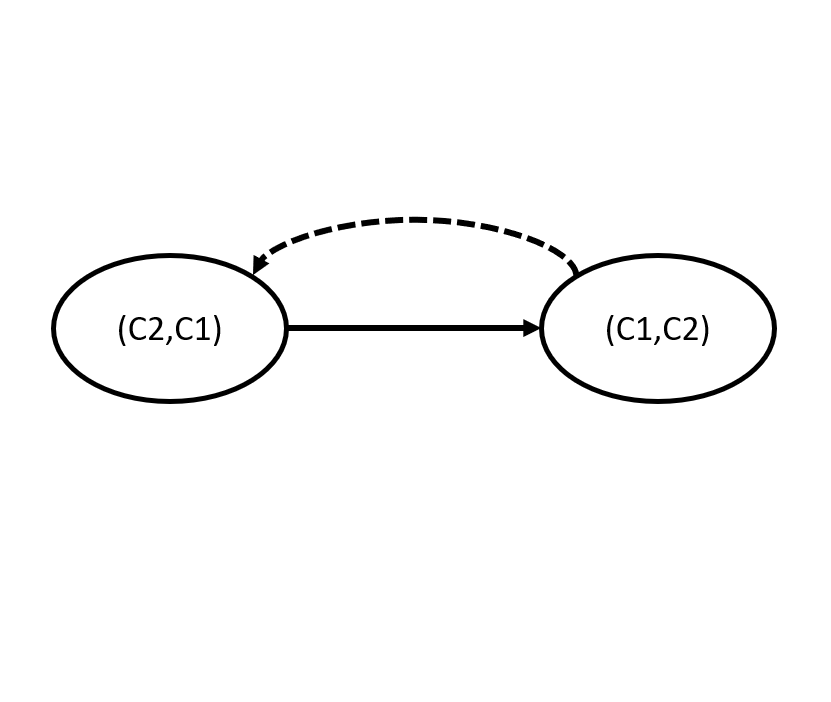}}
	\label{1d}
	\caption{The initial policy and search tree.}
	\label{fig1}
\end{figure}

Let $I$ denote the set of all agents. For the subset $\Omega \subseteq I$, the policies of agents in this subset are denoted as  $\phi^{\Omega}(q^{\Omega})=\prod_{i\in \Omega}\phi^i(q^i)$. When $\Omega=I$, we have $\phi(q)=\prod_{i\in I}\phi^i(q^i)$. Similarly, given that $\pi_{\phi}^{\Omega}(q^{\Omega},q_F^{\Omega})$ denotes the pathes induced by $\phi^{\Omega}$, $\pi_{\phi}(q,q_F)$ are the pathes induced by$\phi$. Let $f(\pi_{\phi}(q_k,q_F))$ denote the cost of the path, we can know that $\forall \pi(q_k,q_F), f(\pi(q_k,q_F))\ge f(\pi_{\phi}(q_k,q_F))$.

When planning its path, each agent $a^i$ believes that $\pi_{\phi}^i(q_k^i,q_F^i)$ will not conflict with other agents' pathes, until encountering a conflict. In subdimensional expansion, $q_k\in Q$ needs to maintain a conflict set $C_k$. $C_k$ includes agents which don't hold the belief that their individual optimal pathes have no conflicts with other pathes at $q_k$. Let $\Pi(q_k)$ denotes all the searched pathes which pass through $q_k$, the function $\Psi()$ calculates the conflicts, then $C_k$ can be defined as:
\begin{equation}
C_k=\bigcup_{\pi\in\Pi(q_k)}\Psi(\pi)
\end{equation}
Thus, agents in the conflicts which are induced by the pathes starting from $q_k$ are included in $C_k$.

For the instance shown in Figure 1(a), the white circles denote the start positions of two agents $a^1$ and $a^2$, while the dark circles represent the goal positions. Figure 1(b) is the initial optimal policy of $a^1$. Figure 1(c) is the initial optimal policy of $a^2$. At time step 1, the root of search tree has only one successor. The search tree is shown in Figure 1(d).

For agents not in $C_k$, i.e., $\bar{C}_k=I\backslash C_k$, they still hold the belief that their individual optimal pathes have no conflicts. Agents in $C_k$ do not hold that belief. Let $Q^{\#}\subseteq Q$ represent the search space constructed by subdimensional expansion. These constraints of $q_k$ can be translated as:
\begin{equation}
T_{qk}Q^{\#}=t^{\bar{C}}(q_k)\times \prod_{i\in C_k}T_{q_k^i}Q^i
\end{equation}

Agents in $\bar{C}_k$ will move at the direction of $t^{\bar{C}}(q_k)$,  which represents the tangential direction of  $\pi_{\phi}^{\bar{C}_k}(q_k^{\bar{C}_k},q_F^{\bar{C}_k})$ at $q_k$. Hence these agents can only follow their individual optimal path to move. Agents in $C_k$ need to do exhaustive search, by trying all the possible directions, $T_{q_k^i}Q^i$. In the instance of Figure 2(a), after the collision of the initial policies of two agents, $a^1$ and $a^2$ need to do exhaustive search. Their new policies are shown in Figure 2(b) and 2(c). The root now has 16 successors and the search tree is shown in Figure 2(d).

\begin{figure}[t]
	\centering\small
	\subfloat[A MAPF instance]{
		\includegraphics[width=0.3\linewidth]{SE1.png}}
	\label{1a} \hfill
	\subfloat[Exhaustive policy of Agent $a^1$]{
		\includegraphics[width=0.3\linewidth]{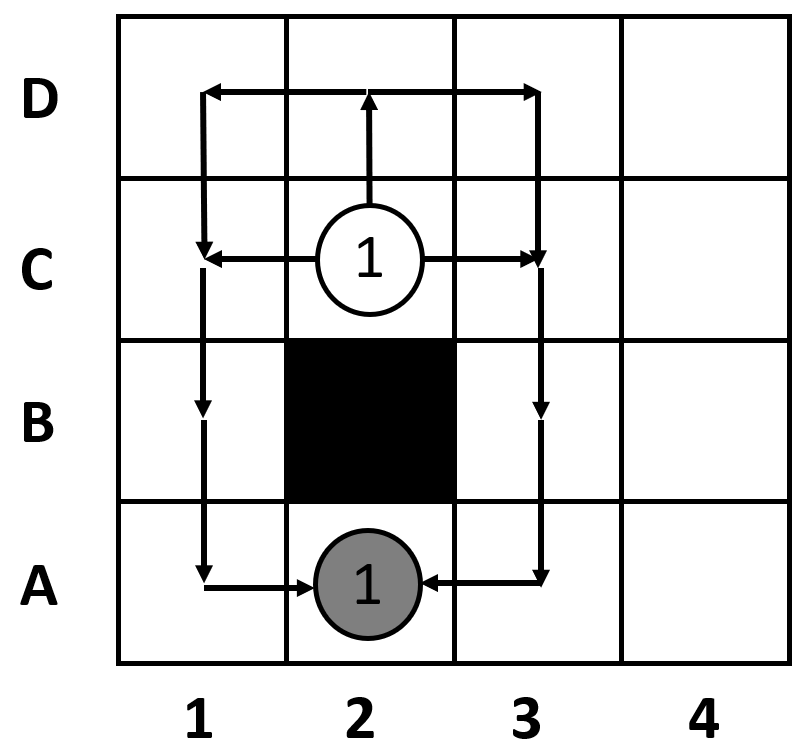}}
	\label{1b}\\
	\subfloat[Exhaustive policy of Agent $a^2$]{
		\includegraphics[width=0.3\linewidth]{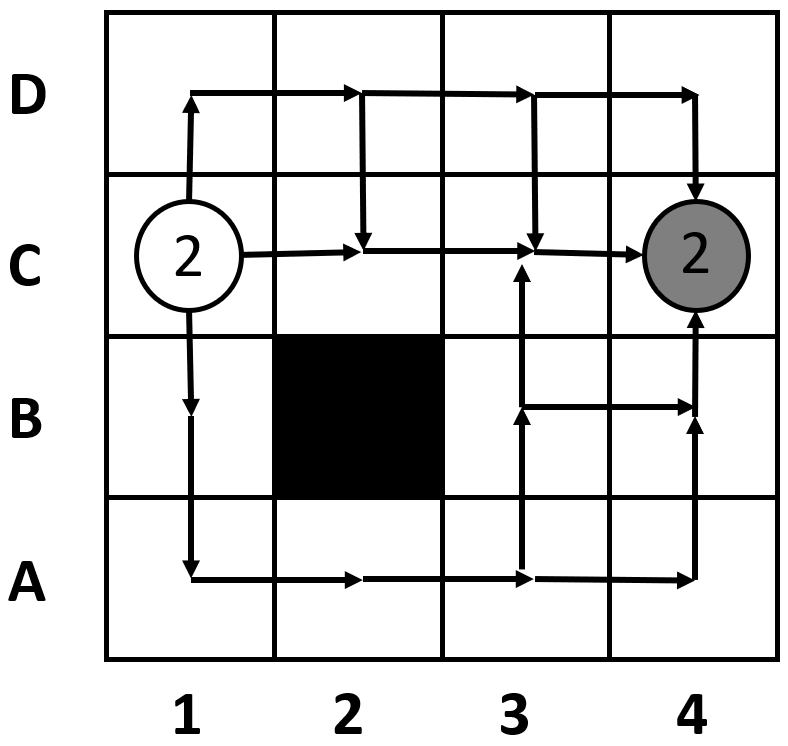}}
	\label{1c}\hfill
	\subfloat[Search tree]{
		\includegraphics[width=0.3\linewidth]{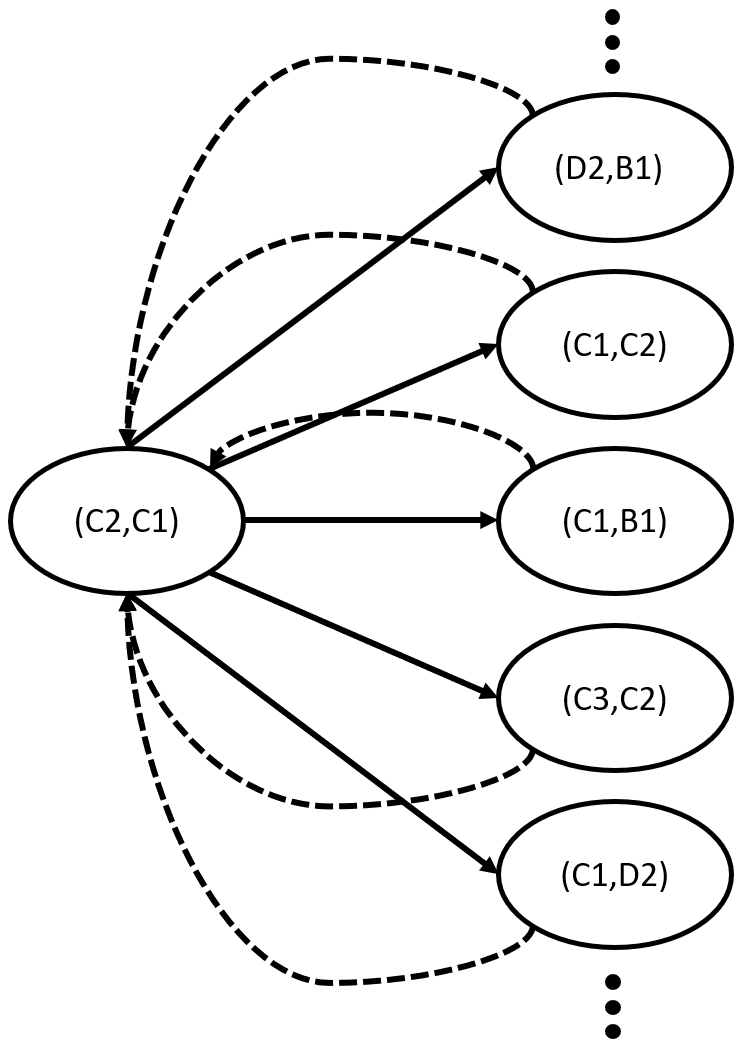}}
	\label{1d}
	\caption{Exhaustive policy and search tree.}
	\label{fig1}
\end{figure}

Starting from $q_I$, by extending $Q^{\#}$, subdimensional expansion will finally finish the construction of $Q^{\#}$. At first, $Q^{\#}$ is equal to $\pi_{\phi}(q_I,q_F)$. During multiple modifications of $Q^{\#}$, the global optimal path will be found.

\subsection{Adding Bypass}

Now, we turn to our method of adding bypass to subdimensional expansion. When $\Psi(\pi_{\phi}(q_k,q_F))\ne \emptyset$, this means there exists at least one collision between two agents $a^i,a^j$ where $i\ne j\in I$. In M*, these two agents will be added into the collision set $C_k$. In contrast, we use bypass to resolve the collisions instead of adding agents into $C_k$ directly.

Given a point $q_k^i\in Q^i$, we have an individual optimal path $\pi_{\phi}(q_{t}^i,q_F^i)$, where $q_t^i$ is the start point of the individually optimal path which pass through $q_k^i$. When this path has conflicts with other agents' individual optimal pathes, if there exists another path $\pi_{\phi'}(q_t^i,q_F^i)\ne \pi_{\phi}(q_t^i,q_F^i)$ and $f(\pi_{\phi'}(q_t^i,q_F^i))=f(\pi_{\phi}(q_t^i,q_F^i))$ and $\pi_{\phi'}(q_t^i,q_F^i)$ has no conflicts with other agents' individual optimal pathes, we can use $\pi_{\phi'}(q_t^i,q_F^i)$ to replace $\pi_{\phi}(q_t^i,q_F^i)$. In this paper, we call this special operation as bypass.

We denote $q_{k'}$ as the predecessor of $q_k$ that has the least cost among all the predecessors of $q_k$. Thus $C_{k'}$ is the collision set of $q_{k'}$. For each collision in $\Psi(\pi_{\phi}(q_k,q_F))$, we can category them into three kind of collisions: 1) Unavoidable Collision (UC), 2) Half-avoidable Collision (HC), and 3) Avoidable Collision (AC). Next, we will explain each of them in more details.

Here, we use UC to denote a collision, in which the two involved agents $a^i,a^j$ are in $C_{k'}$. In a UC, since the two colliding agents are in $C_{k'}$, which means that they are performing exhaustive search at $q_k$, namely, they are not on their individually optimal pathes, we have no need to repair this kind of collisions.

In an HC, one of the two colliding agents $a^i$ is in $C_{k'}$ and another agent $a^j$ is not. Similar to agents in UC, $a^i$ has no need to bypass this collision. For agent $a^j$, it will try to bypass this collision. If $a^j$ can bypass this collision, we will only add $a^i$ into $C_k$ because this collision is avoidable. If $a^j$ fails to bypass it, we will add both of them into $C_k$.

In an AC, both agents $a^i$ and $a^j$ are not in $C_{k'}$. At first, $a^i$ will try to bypass. If $a^i$ or $a^j$ can bypass this collision, so this collision is avoided, we have no need to add them into $C_k$. If both of them fail, so this collision cannot be avoided, we have to add them into $C_k$.

\begin{figure}[t]
	\centering
	\subfloat[A MAPF instance]{
		\includegraphics[width=0.3\linewidth]{SE1.png}}
	\label{1a}\hfill
	\subfloat[Bypass policy of Agent $a^1$]{
		\includegraphics[width=0.3\linewidth]{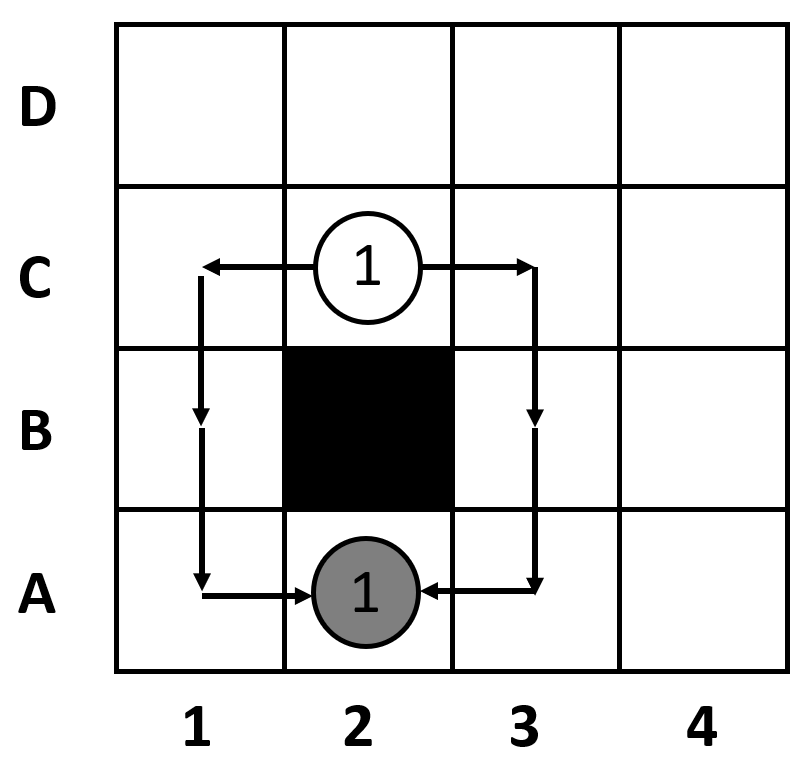}}
	\label{1b}\\
	\subfloat[Bypass policy of Agent $a^2$]{
		\includegraphics[width=0.3\linewidth]{SE4.png}}
	\label{1c}\hfill
	\subfloat[Search tree]{
		\includegraphics[width=0.3\linewidth]{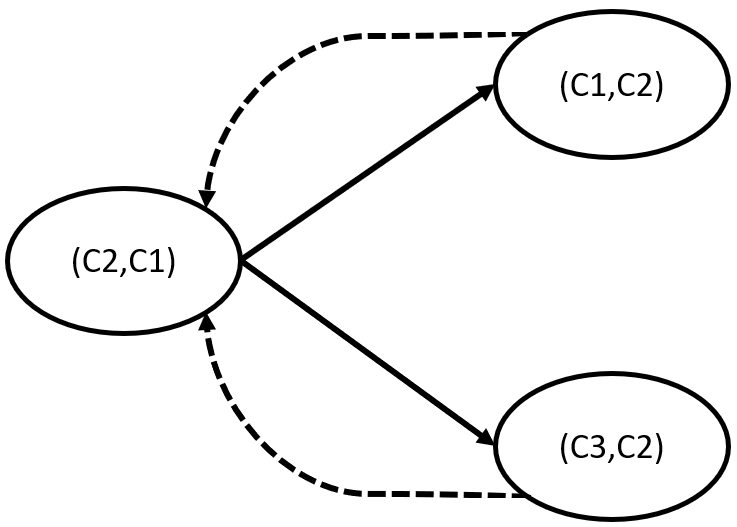}}
	\label{1d}
	\caption{Bypass policy and search tree.}
	\label{fig1}
\end{figure}

For the instance in Figure 3(a), by adding bypass into subdimensional expansion, the two agents have no need to do exhaustive search. Specifically, agent $a^2$ has no need to change its initial policy while agent $a^1$ need to bypass. The root now only has two successors, which is far smaller than the ones in the original subdimensional expansion.

\section{The BPM* Method}

In this section, we propose BPM*, which is an implementation of subdimensional expansion with bypass for solving MAPF. Given a MAPF problem, the map of this instance is denoted as $G$. The search space of each agent $a^i$ is denoted as  $G^i=\{V^i,E^i\}$. $V^i$ represents the set of vertices in $G^i$. $E^i$ represents the set of edges in $G^i$. Thus, the global search space can be denoted as $G=\{V,E\}=\prod_{i\in I}G^i$. Here, $v_I$ represents the start vertex of all agents in $G$ and $v_F$ represents the goal vertex.

\begin{algorithm}[t]\small
	%\SetAlgoLined
	\KwData{$v_k$,$C_l$,open \\
		$v_k$ - vertex in the backpropagation set of $v_l$\\
		$C_l$ - the collision set of $v_l$\\
		open - the open list for BPM*}
	\If{$C_l \nsubseteq C_k$}{$C_k\leftarrow C_k\cup C_l$}
    \If{$\neg(v_k\in open)$}{open.insert($v_k$)}
	\{If the collision set changed, $v_k$ must be re-expanded\}\\
	\For {$v_m\in v_k$.back\_set}
	{\textbf{backprop}($v_m$,$C_k$,open)}
	\caption{backprop}
	\label{algo:algorithm1}
\end{algorithm}

The overall structure of BPM* is similar to M*, while the main difference between them is the method of tackling conflicts. Specifically, in BPM*, conflict set is updated by the following way. When a conflict cannot be resolved by bypass, BPM* will backprop to spread this conflict to its predecessors. The main procedures of backpropagation are shown in Algorithm 1.

\begin{algorithm}[t]\small
	%\SetAlgoLined
	\KwData{$v_k$,$v_l$,open\\
		$v_k$ - vertex in the backpropagation set of $v_l$\\
		open - the open list for BPM*}
	$C_{new} = \emptyset$,
	$new\_path\_set = \emptyset$\\
	\For {$\{a^i,a^j\}\in \Psi(v_l)$}
	{$v_{ki} = $ \textbf{back\_to\_start\_point}($v_k$,$a^i$)\\
		\If {find $\pi_{\phi'}(v_{ki}^i,v_F^i)$ bypasses $\pi_{\phi}(v_{ki}^i,v_F^i)$}
		{$new\_path\_set$.insert($\pi_{\phi'}(v_{ki}^i,v_F^i)$)\\
			\textbf{continue}
		}
		$v_{kj} = $ \textbf{back\_to\_start\_point}($v_k$,$a^j$)\\
		\If {find $\pi_{\phi'}(v_{kj}^j,v_F^j)$ bypasses $\pi_{\phi}(v_{kj}^j,v_F^j)$}
		{$new\_path\_set$.insert($\pi_{\phi'}(v_{kj}^j,v_F^j)$)\\
			\textbf{continue}
		}
		$C_{new}$.insert($\{a^i,a^j\}$)
	}
	\If {$C_{new}==\emptyset$}
	{\textbf{generate\_new\_path}($v_l$,$new\_path\_set$,open)}
	\textbf{return} $C_{new}$
	\caption{bypass}
	\label{algo:algorithm1}
\end{algorithm}

Algorithm 2 shows the implementation of bypass in BPM*. At the beginning of bypassing, we need to maintain a conflict set $C_{new}$ and a new path set $new\_path\_set=\emptyset$. $v_l$ is the successor vertex of $v_k$. For each conflict $\{a^i,a^j\}$ in $\Psi(v_l)$, we will try to use bypass to resolve it. If the corresponding conflict is failed to be resolved, $\{a^i,a^j\}$ will be added into $C_{new}$. When $a^i$ bypassing, it will try to find $v_{ki}$, the starting vertex of the current individual optimal path, by backtracking. Here, the implementation of backtracking is shown in Algorithm 3. When finding $v_{ki}$, we can use single agent cost constrained path planning algorithms to find a bypass for resolving the conflict. If it works, the new path will be added into $new\_path\_set$. After trying to resolve all the conflicts, we need to check whether $C_{new}$ is empty. If empty, only the $new\_path\_set$ should be added into $Q^{\#}$. This process is called $generate\_new\_path$. Finally, $C_{new}$ will be returned.

At the beginning of BPM*, each agent computes its individual optimal policy $\phi^i:V^i\leftarrow V^i$, $\phi^i(v^i)$ represents the successor vertex of $v^i$ which is induced by the policy $\phi^i$. $\phi^i$ can be generated by any kind of pure single agent path finding algorithms. The path induced by $\phi^i$ of $a^i$ can be denoted as $\pi_{\phi}^i(v^i,v_F^i)$. Thus, $\phi:V\leftarrow V$ represents the policy consisted of each agent's individual optimal policy. The generated global path can be represented as $\pi_{\phi}(v,v_F)$. Since the individual optimal path is precomputed, it is suitable for BPM* to take SIC as the heuristic function. The SIC cost of any vertex $v_k$ equals to the sum of the cost of each individual optimal path:
\begin{equation}
f(v_k)=f(\pi_{\phi}(v_k,v_F))\le f(\pi_{*}(v_k,v_F))
\end{equation}

The overall implementation of our BPM* approach is shown in Algorithm 4.

\begin{algorithm}[t]\small
	%\SetAlgoLined
	\KwData{$v_k$,$a^i$}
	$v_t$ = $v_k$.back\_ptr\\
	\If {$v_t$ == $\emptyset$ \textbf{or} $a^i \in C_t$}
	{\textbf{return} $v_k$
	}
	\textbf{\textbf{return} back\_to\_start\_point}($v_t$,$a^i$)
	\caption{back\_to\_start\_point}
	\label{algo:algorithm1}
\end{algorithm}

Considering the bypassing and backpropagation operators, when no conflict exists, BPM* is the same as A* on $G^{sch}$. $G^{sch}$ is the subgraph of $G$, and also the global search space of BPM*. $G^{sch}$ consists of three subgraphs, $G^{exp}$, $G^{nbh}$, $G^{\phi}$. $G^{exp}$ represents the part of $G$ which BPM* has searched. $G^{nbh}$ represents the limited neighbours of vertices in $G^{exp}$. $G^\phi$ represents pathes induced by $\phi$ which start from vertices in $G^{nbh}$ to $v_F$. In BPM*, $G^{exp}$ can be constructed by introducing the $closed\_list$. $G^{nbh}$ and $G^{\phi}$ are implied in $G^{exp}$ and the conflict sets in $G^{exp}$.

\begin{algorithm}[t]\small
	%\SetAlgoLined
	\For {\textbf{all} $v_k\in V$}{ $v_k$.cost $\leftarrow$ MAXCOST,
		$C_k\leftarrow \emptyset$}
	$v_I$.cost $\leftarrow 0$,
	$v_I$.back\_ptr = $\emptyset$,
	open = \{$v_I$\}\\
	\While {open.empty() == False}
	{$v_k \leftarrow$ open.pop() \{Get cheapest vertex\}\\
		\If {$v_k==v_F$}
		{\textbf{return} back\_track($v_k$)\{Reconstruct the optimal path by following the back\_ptr\}}
		\For {$v_l\in V_k^{nbh}$}
		{$v_l$.back\_set.append($v_k$)\{Add $v_k$ to the back propagation list\}\\
			$C_{new}=$ \textbf{bypass}($v_k$,$v_l$,open)\\
			$C_l\leftarrow C_l\cup C_{new}$\\
			\textbf{backprop}($v_k$,$C_l$,open)\\
			\If {$\Psi(v_l)=\emptyset$ \textbf{and} $v_k$.cost+$f(e_{kl})< v_l$.cost}
			{$v_l$.cost $\leftarrow v_k$.cost+$f(e_{kl})$\\
				$v_l$.back\_ptr $\leftarrow v_k$\\
				open.insert($v_l$)
			}
		}
	}
	\textbf{return} No path exists
	\caption{BPM*}
	\label{algo:algorithm1}
\end{algorithm}

$G^{exp}$ includes the vertices which have been added into the $open\_list$. When expanding $v_k\in G^{exp}$, the limited neighbours of $v_k$ are added into the $open\_list$ and also into $G^{exp}$. Edges linking the fore-mentioned vertices are added into $G^{exp}$. Since $G^{exp}$ includes all expanded vertices, it also includes all the searched pathes. For vertex $v_k$, we have:
\begin{equation}
C_k = \left\{
\begin{aligned}
C_{new}^k\bigcup_{v_l\in V_k} C_{new}^l&, \quad v_k\in G^{exp} \\
\emptyset & ,\quad v_k \notin G^{exp}
\end{aligned}
\right.
\end{equation}
Here, $C_{new}^k$ includes agents of conflicts which can not be resolved. $V_k$ represents vertices which have a least one path to $v_k$ in $G^{exp}$. If $v_k\notin G^{exp}$, BPM* has not expanded $v_k$. Thus $C_k$ is also not computed, $C_k=\emptyset$. Therefore, in $G^{sch}$,  there exists one path, in which a conflict vertex $v_k\in G^{sch}\backslash G^{exp}$ exists. This kind of vertex will also be added into the $open\_list$ and $G^{exp}$. After confirming $C_k\ne \emptyset$, the limited neighbours of $v_k$ will be removed from $G^{sch}$.

$G_k^{nbh}$ is a subgraph of $G^{sch}$, representing the limited neighbours of $v_k\in G^{exp}$. $G_k^{nbh}$ includes $v_k$, $V_k^{nbh}$ and corresponding edges. If let $G^{nbh}=\bigcup_{v_k\in G^{exp}}G_k^{nbh}$, $G^{exp}\subset G^{nbh}$. Considering vertices not in $G^{exp}$ and its conflict set$C_k=\emptyset$. The search will start from $v_k\in G^{nbh}\backslash G^{exp}$, following $\pi_{\phi}(v_k,v_f)$, until reaching $v_F$ or any expanded vertex $v_e\in G^{exp}$. This path can be represented as $\pi_{\phi}(v_k)$ or $G_k^\phi$, which can be defined as $G_{\phi}=\bigcup_{v_k\in G^{nbh}\backslash G^{exp}}G_k^{\phi}$.

%All told, the modification of $G^{sch}$ happens only because backpropagation or bypassing. Next is the proof of completeness and optimality of BPM*.

\subsection{Completeness and Optimality}

As aforementioned, completeness and optimality are two key properties for MAPF algorithms. Completeness means that when there exists at least one solution, the algorithm will find a feasible solution in finite time, when there exists no solution, it will eventually terminate. Optimality implies that when the problem is solvable, the solution returned by the algorithm must be optimal. Here, we follow the M* \cite{b8} paper to theoretically prove the completeness and optimality of our BPM* algorithm.

\begin{lemma}
	For an individually optimal path $\pi_{\phi}(v_k^i,v_F^i))$, the distinct bypasses of $\pi_{\phi}(v_k^i,v_F^i))$ is finite.
\end{lemma}
\begin{proof}
	Assume that the distinct bypasses of $\pi_{\phi}(v_k^i,v_F^i))$ is infinite. Let $|V^i|$ denotes the number of vertices in $G^i$. Since $G^i$ is a finite graph, $|V^i|$ is a finite number. We can choose $|V^i|+1$ distinct bypasses of $\pi_{\phi}(v_k^i,v_F^i))$. Since any one of them is different from the other $|V^i|$ bypasses, each bypass $\pi_{\phi}(v_k^i,v_F^i))$ owns at least one unique vertex. Then there exist at least $|V^i|+1$ distinct vertices in $G^i$, which results in a contradiction. Thus the assumption is wrong. Therefore, for an individually optimal path $\pi_{\phi}(v_k^i,v_F^i))$, the distinct bypasses of $\pi_{\phi}(v_k^i,v_F^i))$ is finite.
\end{proof}

For each agent, there exists an upper bound number of bypasses. We denote the upper bound for agent $a^i$ as $U^i$ and the maximum upper bound is denoted as $U=max_{1\le i\le n}U^i$.

\begin{lemma}
	Bypass will not change the cost of the optimal path,  $f(\pi_{\phi'}(v_k,v_F))=f(\pi_{\phi}(v_k,v_F))$.
\end{lemma}
\begin{proof}
	For a vertex $v_k^i\in v_k$, we have an individual optimal path $\pi_{\phi}(v_{k}^i,v_F^i)\in \pi_{\phi}(v_k,v_F)$. When this path has conflicts with other agents' individual optimal pathes, if there exists another path $\pi_{\phi'}(q_k^i,q_F^i)\ne \pi_{\phi}(q_k^i,q_F^i)$ and $f(\pi_{\phi'}(q_k^i,q_F^i))=f(\pi_{\phi}(q_k^i,q_F^i))$ and $\pi_{\phi'}(q_k^i,q_F^i)$ has no conflicts with other agents' individual optimal pathes, we can bypass this conflict. Since $\pi_{\phi'}(q_k^i,q_F^i)$ has no conflicts with other agents' individual optimal pathes, the other path in $\pi_{\phi}(v_k,v_F)$ will not be influenced and their cost will not be changed. Considering $f(\pi_{\phi'}(q_k^i,q_F^i))=f(\pi_{\phi}(q_k^i,q_F^i))$, thus the cost of each agent's path is not changed. Therefore, bypass will not change the cost of the optimal path,  $f(\pi_{\phi'}(v_k,v_F))=f(\pi_{\phi}(v_k,v_F))$.
\end{proof}

\begin{lemma}
	If no solution exists, BPM* will terminate in finite time without returning a path.
\end{lemma}
\begin{proof}
	Assuming no solution exists. As part of BPM*, A* is run on the search graph $G^{sch}$. A* will explore all of $G^{sch}$ in finite time and conclude that no solution exists, except if the A* search is interrupted by a modification of $G^{sch}$. $G^{sch}$ is modified when the collision set of at least one vertex in $G^{sch}$ is changed or one bypass succeeds. We assume that $k(0=<k<=n,k\ne 1)$ agents failed to bypass(which means that they need to be added to the collision set), thus at most $n-k$ agents succeed to bypass(the other agents are not involved in collisions). When at least one agent succeeds to bypass, the modification adds a new path in $G^{sch}$. For each vertex $v$, this kind of modification can happens at most $\sum_{1\le i\le n}U$ times. When at least one agent failed to bypass, the modification adds one or more robots to the collision set, thus each collision set can be modified at most $n-1$ times, which means all of agents cannot bypass. The first modification after bypass failure must add at least two agents. Therefore, $G^{sch}$ can be modified at most $[(\sum_{1\le i\le n-k}U)+(k-1)]*|V|\le [(\sum_{1\le i\le n}U)+(n-1)]*|V|$ times. Thus if no solution exists, BPM* will always terminate in finite time.\par
	We now show that BPM* will never return an invalid path containing a robot-robot collision. A vertex $v_k$ has out-neighbours only if it is collision free, unless $v_k$ is not in the explored graph $G^{exp}$. Before BPM* will return a path passing through $v_k$, $v_k$ must be added to the open list, and thus to $G^{exp}$. When $v_k$ which is not collision free is added to the open list, $G^{sch}$ is modified to remove all out-neighbours of $v_k$, which removes any path passing through $v_k$ from $G^{sch}$. Therefore, BPM* will never return a path passing through a state at which agents collide. Thus, if no solution exists, BPM* will terminate in finite time without a path.
\end{proof}

\begin{lemma}
If an optimal path exists, BPM* will find the optimal path in finite time if one of two cases always hold:\par
\noindent\textbf{Case 1:}The search graph $G^{sch}$ contains an optimal path, $\pi_{*}(v_s,v_f)$. \par
\noindent\textbf{Case 2:}The search graph $G^{sch}$ contains a path $\pi(v_s,v_c)$ such that $g(\pi(v_s,v_c))+h(v_c)\le g(\pi_{*}(v_s,v_f))$, and $\exists v_b\in \pi(v_s,v_c)$ such that $C_{new}^c\nsubseteq C_b$.
\end{lemma}

Here, Case 2 implies the existence of a path which has not been explored by BPM* that leads to a robot-robot collision at $v_c$, and which costs no more than $\pi_{*}(v_s,v_f)$. This situation can happen only because the change of the collision set rather than bypass, since bypass does not change the cost of a path (Lemma 2). If the path had been explored, $v_b$ and $v_c$ would have been added to the open list and thus to the explored graph $G^{exp}$. In this case, $C_b$ would include all robots involved in the collision at $v_c$, i.e. the robots in $C_{new}^c$.
To prove Lemma 2, we proceed by showing that if case 1 holds, the optimal path will be found unless a cheaper path containing a collision exists in the search graph $G^{sch}$, i.e., Case 2 holds (Lemma 5). We then show that BPM* will never explore a suboptimal path to the goal as long as Case 2 holds (Lemma 6), and that Case 2 will not hold after finite time (Lemma 7). We conclude by proving that either Case 1 or Case 2 will always hold, demonstrating that the optimal path will be found (Lemma 9).

\begin{lemma}
	If the search graph $G^{sch}$ contains an optimal path(i.e. case 1 holds), BPM* will find the optimal path, unless case 2 also holds.
\end{lemma}
\begin{proof}
	If Case 1 holds, running A* on $G^{sch}$ will find $\pi_{*}(v_s,v_f)$ in finite time, unless there exists a cheaper path $\pi_{cheaper}(v_s,v_f)\subseteq G^{sch}$, which we now show would satisfies the conditions for case 2 to hold. Because $\pi_{*}(v_s,v_f)$ is a minimal cost collision-free path, $\pi_{cheaper}(v_s,v_f)$ must contain an agent-agent collision. Therefore a vertex $v_k\in \pi_{cheaper}(v_s,v_f)$ must exist such that $C_{new}^c\ne \emptyset$, and $g(\pi_{cheaper}(v_s,v_k))+h(v_k)<g(\pi_{*}(v_s,v_f))$. The existence of a path through $v_k$ implies that $v_k \notin G^{exp}$, as a vertex containing agent-agent collisions has its outneighbors removed when added to the explored graph $G^{exp}$. Therefore, $C_k=\emptyset$. Since $C_{new}^c \nsubseteq C_k$, $v_k$ fulfills the roles of both $v_b$ and $v_c$ in the definition of Case 2. As a result, if Case 1 holds, BPM* will find $\pi_{*}(v_s,v_f)$, unless Case 2 also holds.
\end{proof}

\begin{lemma}
	If the search graph $G^{sch}$ contains an unexplored path cheaper than $g(\pi_{*}(v_s,v_f))$(i.e. case 2 holds), BPM* will not return a suboptimal path.
\end{lemma}
\begin{proof}
	If Case 2 holds, then $\pi(v_s,v_c)$ will be explored by A* and added to the explored graph $G^{exp}$ before A* finds any path to $v_f$ that costs more than $g(\pi_{*}(v_s,v_f))$. Adding $\pi(v_s,v_c)$ to $G_{exp}$ will modify $C_b$. $G^{sch}$ will then be modified to reflect the new limited neighbours of $v_b$ and A* will be restarted. Therefore, BPM* will never return a suboptimal path as long as Case 2 holds.
\end{proof}

\begin{lemma}
	The search graph $G^{sch}$ will cease to contain any unexplored path cheaper $g(\pi_{*}(v_s,v_f))$(i.e. Case 2 will cease to hold) after finite time.
\end{lemma}
\begin{proof}
	For Case 2 to hold, there must be at least one vertex $v_b$ such that $C_b$ is a strict subset of I. $G^{sch}$ can be modified at most $[(\sum_{1\le i\le n}U)+(n-1)]*|V|$ times before all collision sets are equal to I. Therefore, after a finite number of modifications of $G^{sch}$ Case 2 cannot hold. A* will fully explore any finite graph in finite time, implying that the time between any two successive modifications of $G^{sch}$ is finite. Therefore, Case 2 will not hold after finite time.
\end{proof}

With these auxiliary results in hand, the proof of Lemma 4 is as follows. If Case 1 holds, then BPM* will find the optimal path in finite time, unless case 2 also holds (Lemma 5). While Case 2 holds, BPM* will not return a suboptimal path (Lemma 6), and case 2 cannot hold after finite time (Lemma 7). Therefore, after finite time, only Case 1 will hold, implying that BPM* will find the optimal path in finite time.
To complete the proof of the completeness and optimality of BPM*, we must show that Case 1 or Case 2 will always hold. To do so, we first need an auxiliary result (Lemma 8) showing that the optimal path for some subset of robots costs no more than the joint path taken by those robots in the optimal, joint path for the entire set of robots. The auxiliary result is used to demonstrate that an optimal path can be found by combining optimal pathes for disjoint subsets of robots.

Let $\pi_{\Omega}^{'}(v_k,v_f)$ be the path constructed by combining the optimal path for a subset $\Omega \subset I$ of robots with the individually optimal pathes for the robots in $I\backslash\Omega$.

\begin{lemma}
	If the configuration graph contains an optimal path $\pi_{*}(v_k,v_f)$, then $\forall\Omega\subset I$, $g(\pi_{\Omega}^{'}(v_k,v_f))\le g(\pi_*(v_k,v_f))$. Furthermore, if $\Omega_{1} \subset \Omega_{2}$, then $g(\pi_{\Omega_{1}}^{'}(v_k,v_f))\le g(\pi_{\Omega_{2}}^{'}(v_k,v_f))$.
\end{lemma}
\begin{proof}
	If $\pi_*(v_k,v_f)$ from an arbitrary $v_k$ to $v_f$ exists in G, then for any subset of robots $\Omega$ there exists an optimal path $\pi_*^{\Omega}(v_k^{\Omega},v_f^{\Omega})$ which costs no more than the path taken by those robots in $\pi_*(v_k,v_f)$. Let $\bar{\Omega}=I\backslash \Omega$ be the complement of $\Omega$ and $\pi_{\emptyset}^{\bar{\Omega}}(v_k^{\bar{\Omega}},v_f^{\bar{\Omega}})$ costs no more than the pathes taken by the robots in $\bar{\Omega}$ in $\pi_*(v_k,v_f)$ by the construction of the individual policies. A path for all robots in I, $\pi_{\Omega}^{'}(v_k,v_f)$, is then constructed by having each robot in $\Omega$ follow its path in $\pi_*^{\Omega}(v_k^{\Omega},v_f^{\Omega})$, while each robot in $\bar{\Omega}$ follows its path in $\pi_{\emptyset}^{\bar{\Omega}}(v_k^{\bar{\Omega}},v_f^{\bar{\Omega}})$. Since the individual path for each robot in $\pi_{\Omega}^{'}(v_k,v_f)$ costs no more than the path for the same robot in $\pi_*{v_k,v_f}$, $g(\pi_{\Omega}^{'}(v_k,v_f))\le g(\pi_*(v_k,v_f))$. By the same logic, if $\Omega_{1} \subset \Omega_{2}$, then $g(\pi_{\Omega_{1}}^{'}(v_k,v_f))\le g(\pi_{\Omega_{2}}^{'}(v_k,v_f))$.
\end{proof}

\begin{lemma}
	The search graph $G^{sch}$ will always contain an optimal path (i.e. Case 1 will hold) or an unexplored path which costs no more than the optimal path (i.e. Case 2 will hold) at all points in the execution of BPM*.
\end{lemma}
\begin{proof}
We proceed by showing that the limited neighbours of each vertex in $G^{sch}$ are sufficient to construct either the optimal path, or some unexplored, no more expensive path. Consider the vertex $v_k\in G^{sch}$ with collision set $C_k$. The successor of $v_k$ in $\pi_{C_k}^{'}(v_k,v_f)$, $v_l$, is a limited neighbour of $v_k$ by the definition of the limited neighbours. Since $C_l\subseteq C_k$, Lemma 8 implies
	\begin{equation}
	\begin{aligned}
	g(\pi_{C_k}^{'}(v_k,v_l))+g(\pi_{C_l}^{'}(v_l,v_f))\le \\
	g(\pi_{C_k}^{'}(v_k,v_f)) \le g(\pi_{*}(v_k,v_f))
	\end{aligned}
	\end{equation}

We apply the above bound vertex by vertex from the initial vertex to show that a path $\pi^{''}(v_s,v_f)\in G^{sch}$ can be constructed which satisfies either Case 1 or Case 2. The successor of the $m'$th vertex in $\pi^{''}(v_s,v_f)$ is the successor of $v_m$ in $\pi_{C_m}^{'}(v_m,v_f)$. Applying (5) gives the bound $g(\pi^{''}(v_s,v_f))\le g(\pi_{C_s}(v_s,v_f))\le g(\pi_*(v_s,v_f))$. If $\pi^{''}(v_s,v_f) = \pi_*(v_s,v_f)$ then Case 1 is satisfied. Otherwise, there is a vertex $v_c \in \pi^{''}(v_I,v_F)$ such that $C_{new}^c\ne \emptyset$. Let $v_b$ be the predecessor of $v_c$, which implies that $v_c$ lies in  $\pi_{C_b}^{'}(v_b,v_F)$. Because by construction the agents in $C_b$ do not collide with one another in  $\pi_{C_b}^{'}(v_b,v_F)$, $C_{new}^c\nsubseteq C_b$. Since $f(\pi^{''}(v_I,v_c))+h(v_c,v_F)\le f(\pi^{''}(v_I,v_F))\le f(\pi_*(v_I,v_F))$, Case 2 is satisfied.

If Case 1 does not hold, there is an edge case which must be considered. If $\pi^{''}(v_I,v_F)$ contains a vertex $v_k\notin G^{exp}$ with a successor $v_l\in G^{exp}$, $C_l$ may not be a subset of $C_k$, because no path exists from $v_k$ to $v_l$ in the explored graph $G^{exp}$, so the bound given by (5) does not apply. However, in this case the path induced by $\phi$ from $v_l$ must terminate at some vertex $v_c$ with $C_{new}^c\ne \emptyset$. We construct a new path by following $\pi^{''}(v_I,v_F)$ to $v_l$, and then following $\pi_{\phi}(v_l,v_F)$. The sum of the cost of this path and $h(v_c,v_F)$ must be less than $f(\pi_*(v_I,v_F))$, and $C_{new}^c\nsubseteq C_k$, so Case 2 still holds.

Therefore, the search graph $G^{sch}$ will always contain an optimal path (i.e. Case 1 will hold) or an unexplored path which costs no more than the optimal path (i.e. Case 2 will hold) at all points in the execution of BPM*.
\end{proof}

\begin{theorem}
	BPM* is complete and optimal.
\end{theorem}
\begin{proof}
If the configuration graph $G$ does not contain an optimal path, then BPM* will terminate in finite time without returning an invalid path(Lemma3). If $G$ does contain an optimal path, then the search graph must always contain either the optimal path, or an unexplored path which costs no more than the optimal path (Lemma 9), which implies that then BPM* will find the optimal path in finite time (Lemma 4). BPM* will thus find the optimal path in finite time, if one exists, or terminate in finite time if no path exists. Therefore, BPM* is complete and optimal.
\end{proof}

\section{Experiments}

Notice that both M* and BPM* are complete and optimal. This means that both algorithms will output the optimal solution if exists or no solution if infeasible for a MAPF problem given sufficient amount of time. The goal of our experiments is to evaluate the efficiency of our algorithm in solving large MAPF problems. To this end, we conducted experiments on the following benchmark problems.

In these problems, the map is set as a grid map, and agents can only move from the current node to their four neighbor nodes. When generating problem instances, the density of agents for each given number of agents is guaranteed to be an average of one agent per 100 grid points, so that the number of grid points in the map is determined. Because the agent density of all the generated problems is the same, we only need to analyze the impact of the number of agents on the performance of the tested algorithms without considering the impact of the density. When a grid map is generated, each grid point has a 20\% independent probability of turning into an obstacle. Given a problem instance, each agent's starting position grid point and goal position grid point are randomly selected. Additionally, we ensured that the selected grid points are non-obstacle grid point and there exists at least one feasible path between each agent's starting position and goal position. We tested on this domain with constant agent density that contains 21 groups of problems, where each group consists of 100 problems.

Similar to M* and rM* \cite{b8}, we further improved BPM* by recursion and produced rBPM*. We ran all the algorithms on the same problem instances and reported the success rates and the solving speed. Here, the upper limit of time to solve a MAPF instance is set to be 5 minutes. All algorithms are implemented in C++ and run on a PC with an i5-6300hq CPU, with 2.30Ghz CPU frequency and 8GB memory size.

\begin{figure}[t]
	\centering
	\subfloat[Success rate]{
		\includegraphics[width=0.45\linewidth]{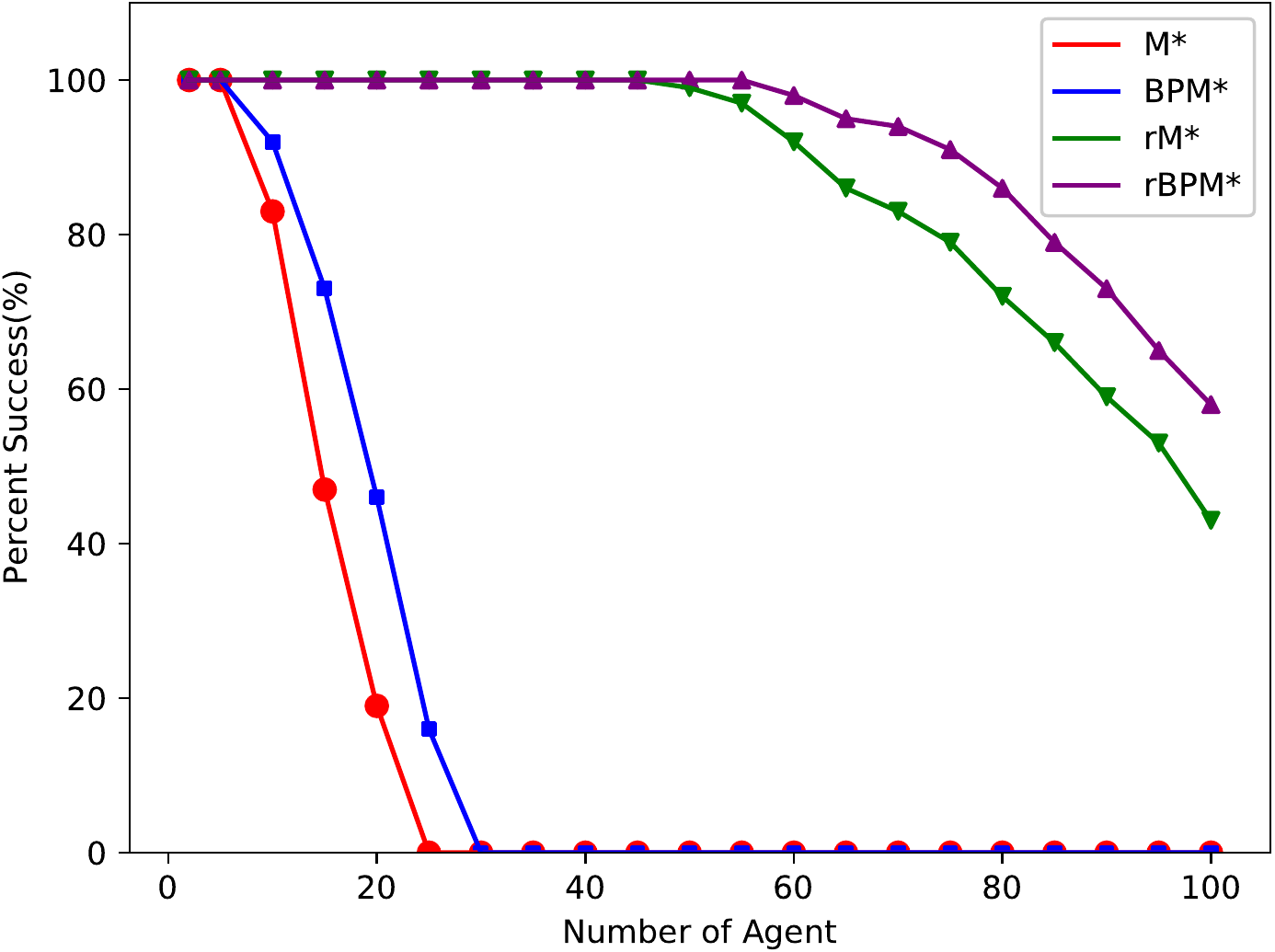}}
	\label{1a}\hfill
	\subfloat[Median time]{
		\includegraphics[width=0.45\linewidth]{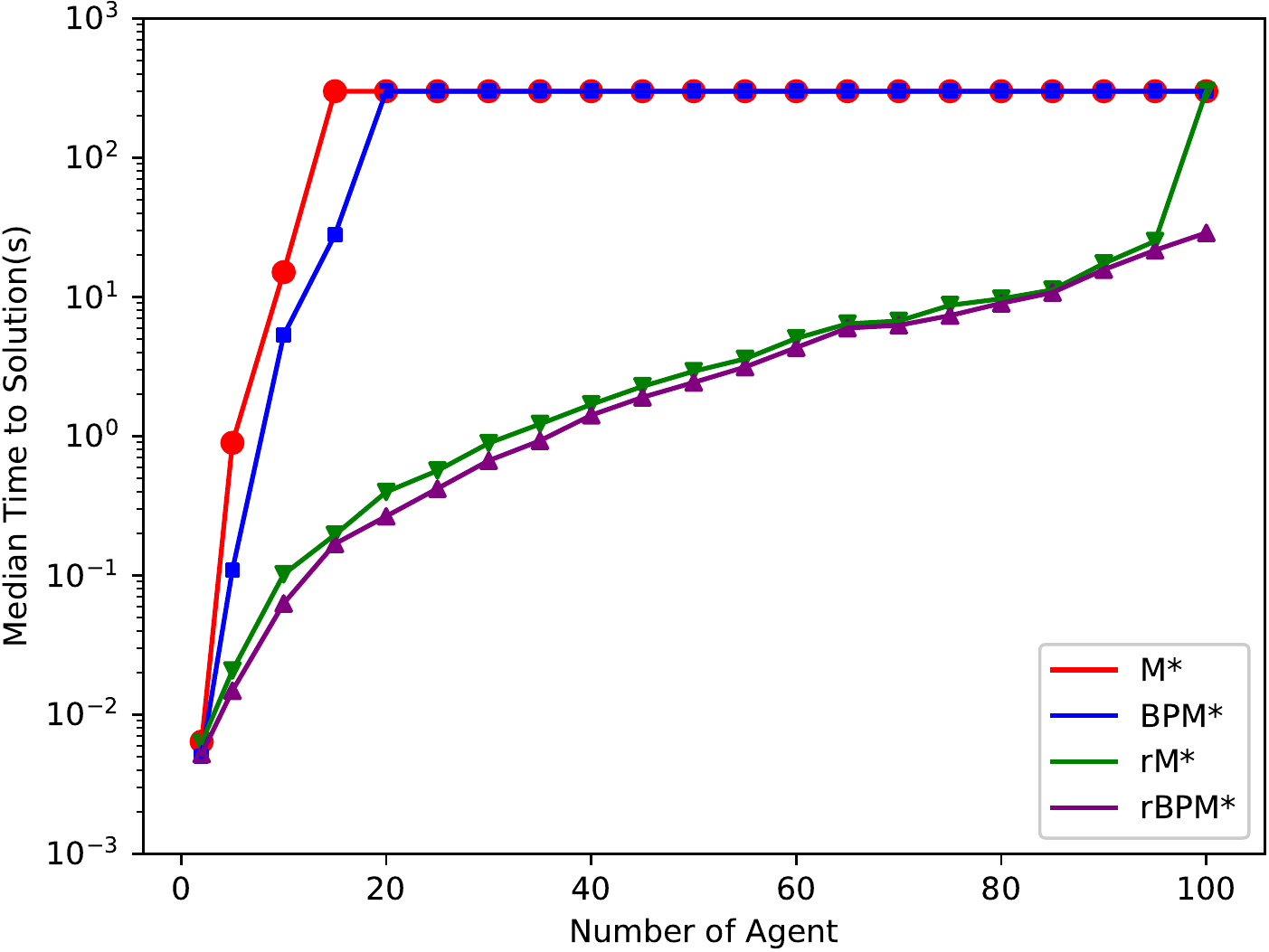}}
	\label{1b}\\
	\caption{Experimental results.}
	\label{fig1}
\end{figure}

As shown by Figure 4, BPM* outperforms M* in terms of the success rate especially for large problems with many agents. In terms of solving time, BPM* also showed competitive performance. For the variants with recursion, both rM* and rBPM* gained significant improvements comparing to M* and BPM* respectively. However, rBPM* still achieved better performance than rM*. This confirms the advantage of our approach comparing to the leading MAPF method.

\section{Conclusions}

This paper proposes BPM*, which is an implementation of subdimensional expansion with bypass. We proved that BPM* is complete and optimal, which are critical for solving MAPF problems. We conducted experiments on several MAPF problems to evaluate the efficiency of BPM*. Our experiments confirm the advantage of BPM* comparing to the state-of-the-art. As shown in the experiments, BPM* can be improved by recursion. In future work, we plan to integrate BPM* with other MAPF techniques and test our approach on large real-world applications.

\bibliography{citations}

\begin{thebibliography}{8}
\providecommand{\natexlab}[1]{#1}
\providecommand{\url}[1]{\texttt{#1}}
\providecommand{\urlprefix}{URL }
\expandafter\ifx\csname urlstyle\endcsname\relax
  \providecommand{\doi}[1]{doi:\discretionary{}{}{}#1}\else
  \providecommand{\doi}{doi:\discretionary{}{}{}\begingroup
  \urlstyle{rm}\Url}\fi

\bibitem[{Goldenberg et~al.(2014)Goldenberg, Felner, Stern, Sharon, and
  Schaeffer}]{b6}
Goldenberg, M.; Felner, A.; Stern, R.; Sharon, G.; and Schaeffer, J. 2014.
\newblock Enhanced Partial Expansion A*.
\newblock \emph{Journal of Artificial Intelligence Research} 50(1): 141--187.

\bibitem[{Standley(2010)}]{b5}
Standley, T. 2010.
\newblock Finding Optimal Solutions to Cooperative Pathfinding Problems.
\newblock In \emph{Proc. of AAAI}.

\bibitem[{Sturtevant(2010)}]{b2}
Sturtevant, N.~R. 2010.
\newblock A Comparison of High-Level Approaches for Speeding Up Pathfinding.
\newblock In \emph{The 6th Int. Conf. AAAI Conference on Artificial
  Intelligence and Interactive Digital Entertainment}.

\bibitem[{Surynek(2010)}]{b4}
Surynek, P. 2010.
\newblock An Optimization Variant of Multi-Robot Path Planning Is Intractable.
\newblock In \emph{Proc. of AAAI}.

\bibitem[{Veloso et~al.(2015)Veloso, Biswas, Coltin, and Rosenthal}]{b3}
Veloso, M.; Biswas, J.; Coltin, B.; and Rosenthal, S. 2015.
\newblock CoBots: robust symbiotic autonomous mobile service robots.
\newblock In \emph{Proc. IJCAI}.

\bibitem[{Wagner and Choset(2015)}]{b8}
Wagner, G.; and Choset, H. 2015.
\newblock Subdimensional Expansion for Multirobot Path Planning.
\newblock \emph{Artificial Intelligence} 219: 1--24.

\bibitem[{Wagner and G.~Wagner(2011)}]{b7}
Wagner, G.; and G.~Wagner, H.~C. 2011.
\newblock M*: A Complete Multirobot Path Planning Algorithm with Performance.
\newblock In \emph{Proc. of IROS}, 3260--3267.

\bibitem[{Wurman, D'Andrea, and Mountz(2007)}]{b1}
Wurman, P.~R.; D'Andrea, R.; and Mountz, M. 2007.
\newblock Coordinating Hundreds of Cooperative, Autonomous Vehicles in
  Warehouses.
\newblock In \emph{Proc. AAAI}.

\end{thebibliography}

\end{document}